\documentclass{tlp}

\usepackage{mathptmx}
\usepackage{url}
\usepackage{caption}

\newcommand{\true}{{\bf true}}
\newcommand{\false}{{\bf false}}

\usepackage{verbatim}
\usepackage[all]{xy}
\usepackage{pgf, pgfnodes, pgfarrows}
\usepackage[sumlimits]{amsmath}
\usepackage{amssymb}
\usepackage{amstext}
\usepackage{amscd}

\newtheorem{lemma}{Lemma}
\newtheorem{theorem}{Theorem}
\newtheorem{bv12irrelevant}{BV12-Irrelevant}
\newtheorem{bv12intrinsic}{BV12-Intrinsic}
\newtheorem{hall07irrelevant}{Hall07-Irrelevant}
\newtheorem{hall07intrinsic}{Hall07-Intrinsic}
\newtheorem{dependenceirrelevant}{Dependence-Irr}
\newtheorem{dependenceintrinsic}{Dependence-Intr}
\newtheorem{productionirrelevant}{Production-Irr}
\newtheorem{productionintrinsic}{Production-Intr}

\newtheorem{definition}{Definition} 

\begin{document}
\bibliographystyle{acmtrans}

\long\def\comment#1{}

\title{Towards a General Framework for Actual Causation Using CP-logic}

\author[S. Beckers and J. Vennekens]
{SANDER BECKERS and JOOST VENNEKENS \\
KULeuven - University of Leuven\\
\{Sander.Beckers,Joost.Vennekens\}@cs.kuleuven.be
}

\pagerange{\pageref{firstpage}--\pageref{lastpage}}
\volume{\textbf{10} (3):}
\jdate{March 2002}
\setcounter{page}{1}
\pubyear{2002}

\maketitle

\label{firstpage}

\begin{abstract}
Since Pearl's seminal work on providing a formal language for causality, the subject has garnered a lot of interest among philosophers and researchers in artificial intelligence alike. One of the most debated topics in this context is the notion of actual causation, which concerns itself with specific -- as opposed to general -- causal claims. The search for a proper formal definition of actual causation has evolved into a controversial debate, that is pervaded with ambiguities and confusion. The goal of our research is twofold. First, we wish to provide a clear way to compare competing definitions. Second, we want to improve upon these definitions so they can be applied to a more diverse range of instances, including non-deterministic ones. To achieve these goals we provide a general, abstract definition of actual causation, formulated in the context of the expressive language of CP-logic (Causal Probabilistic logic). We will then show that three recent definitions by Ned Hall (originally formulated for structural models) and a definition of our own (formulated for CP-logic directly) can be viewed and directly compared as instantiations of this abstract definition, which also allows them to deal with a broader range of examples.

\end{abstract}

\section{Introduction}

Suppose we know the causal laws that govern some domain, and that we then observe a story that takes place in this domain; when should we now say that, in this particular story, one event caused another? Ever since Lewis \shortcite{lewis73} first analyzed this problem of actual causation (a.k.a. token causation) in terms of counterfactual dependence, philosophers and researchers from the AI community alike have been trying to improve on his attempt. Following \cite{pearl:book}, structural equations have become a popular formal framework for this \cite{hitchcock07,hitchcock09,halpernpearl05a,hh14}. A notable exception is the work of Ned Hall, who has extensively critiziced the privileged role of structural equations for causal modelling, as well as the definitions that have been expressed with it. He has proposed several definitions himself \shortcite{hall03,hall04,hall07}, the latest of which is a sophisticated attempt to overcome the flaws he observes in those that rely too heavily on structural equations. We have developed a definition of our own in \cite{beckers,vennekens11}, within the framework of CP-logic (Causal Probabilistic logic). 

The relation between these different approaches is currently not well understood. Indeed, they are all expressed using different formalisms (e.g., neuron diagrams, structural equations, CP-logic, or just natural language). Therefore, comparisons between them are limited to verifying on which examples they (dis)agree. In this paper, we work towards a remedy for this situation. We will present a general, parametrized definition of actual causation in the context of the expressive language of CP-logic. Exploiting the fact that neuron diagrams and structural equations can both be reduced to CP-logic, we will then show that our definition and three definitions by Ned Hall can be seen as particular instantiations of this parametrized definition. This immediately provides a clear, conceptual picture of the similarities and differences between these approaches. Our analysis thus allows for a formal and fundamental comparison between them.

This general framework for comparing different approaches to actual causation is the main contribution of this paper. In addition, placing existing approaches in this framework may make it easier to improve/extend them. Our versions of Hall's definitions illustrate this, as their scope is expanded to also include non-deterministic examples, and cases of causation by omission. Further, our formulations prove to be simpler than the original ones and their application becomes more straightforward. While our ambition is to work towards a framework that encompasses a large variety of approaches to actual causation, this goal is obviously infeasible within the scope of a single paper. We have therefore chosen to focus most of our attention on Hall, because his work is both among the most refined and most influential in this field; in addition, it is also representative for a larger body of work in the counterfactual tradition.

We first introduce the CP-logic language in Section \ref{sec:cp}. In Section \ref{sec:def}, a general definition of actual causation is first presented, and then instantiated into four concrete definitions. Section \ref{sec:summary} offers a succinct representation of all these definitions, and an illustration of how they compare to each other.

\section{CP-logic}\label{sec:cp}

We give a short, informal introduction to CP-logic. A detailed description can be found in \cite{vennekens09}. The basic syntactical unit of CP-logic is a CP-law, which takes the general form $Head \leftarrow Body$. The body can in general consist of any first-order logic formula. However, in this paper, we restrict our attention to conjunctions of ground literals. The head contains a disjunction of atoms annotated with probabilities, representing the possible effects of this law. When the probabilities in a head do not add up to one, we implicitly assume an {\em empty} disjunct, annotated with the remaining probability. 

Each CP-law models a specific {\em causal mechanism}. Informally, if the $Body$ of the law is satisfied, then at some point it will be applied, meaning one of the disjuncts in the $Head$ is chosen, each with their respective probabilities. If a disjunct is chosen containing an atom that is not yet $\true$, then this law causes it to become $\true$; otherwise, the law has no effect. A finite set of such CP-laws forms a CP-theory, and represents the causal structure of the domain at hand. The domain unfolds by laws being applied one after another, where multiple orders are often possible, and each law is applied at most once. We illustrate with an example from \cite{hall04}:
\begin{quote}
Suzy and Billy each decide to throw a rock at a bottle. When Suzy does so, her rock shatters the bottle with probability $0.9$. Billy's aim is slightly worse and he only hits with probability $0.8$.
\end{quote}
This small causal domain can be expressed by the following CP-theory $T$:

\begin{minipage}{0.4\textwidth}
\begin{align} 
Throws(Suzy) &\leftarrow.\label{suzy1}\\
Throws(Billy) &\leftarrow.
\end{align}
\end{minipage}
\begin{minipage}{0.5\textwidth}
\begin{align}
(Breaks: 0.9) &\leftarrow Throws(Suzy).\label{suzy}\\
(Breaks:0.8) &\leftarrow Throws(Billy).\label{billy}
\end{align}
\end{minipage}
\\

The first two laws are {\em vacuous} (i.e., they will be applied in every story) and {\em deterministic} (i.e., they have only one possible outcome, where we leave implicit the probability $1$). The last two laws are {\em non-deterministic}, causing either the bottle to break or nothing at all.

The given theory summarizes all possible {\em stories} that can take place in this model. For example, it allows for the story consisting of the following chain of events:

\begin{quote}
Suzy and Billy both throw a rock at a bottle. Suzy's rock gets there first, shattering the bottle. However Billy's throw was also accurate, and would have shattered the bottle had it not been preempted by Suzy's throw.
\end{quote}
To formalize this idea, the semantics of CP-logic uses {\em probability trees} \cite{shafer:book}. For this example, one such tree is shown in Figure \ref{fig:tree}. Here, each node represents a state of the domain, which is characterized by an assignment of truth values to the atomic formulas, in this case $Throws(Suzy)$, $Throws(Billy)$ and $Breaks$. In the initial state of the domain (the root node), all atoms are assigned their {\em default} value $\false$. In this example, the bottle is initially unbroken and the rocks are still in Billy and Suzy's hands. The children of a node $x$ are the result of the application of a law: each edge $(x,y)$ corresponds to a specific disjunct that was chosen from the head of the law that was applied in node $x$. In this particular case, law \eqref{suzy1} is applied first, so the assignment in the child-node is obtained by setting $Throws(Suzy)$ to $\true$, its {\em deviant} value. The third state has two child-nodes, corresponding to law \eqref{suzy} being applied and either breaking the bottle (left child) or not (right child). The leftmost branch is thus the formal counterpart of the above story, where the last edge represents the fact that Billy's throw was also accurate, even though there was no bottle left to break. A branch ends when no more laws can be applied.
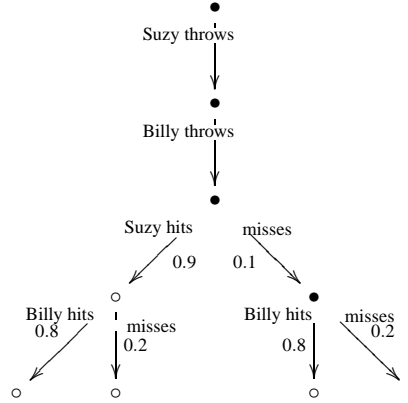
\begin{figure}
\centerline{
\xymatrix@H=0.2cm@C=3em{
&&\bullet \ar[d]|(.3){\text{Suzy throws\hspace*{0.7cm}}}\\ 
&&\bullet \ar[d]|(.3){\text{Billy throws\hspace*{0.7cm}}}\\ 
&&\bullet \ar[ld]^{0.9}|(.3){\text{Suzy hits\hspace*{0.7cm}}} \ar[rd]_{0.1}|(.3){\text{\hspace*{0.6cm}misses}}\\
&\circ \ar[ld]_{0.8}|(.2){\text{Billy hits\hspace*{0.95cm}}} \ar[d]^{0.2}|(.3){\hspace*{0.95cm}\text{misses}} && \bullet \ar[d]_{0.8}|(.2){\text{Billy hits}\hspace*{0.95cm}} \ar[rd]^{0.2}|(.2){\hspace*{0.95cm}\text{misses}} \\
\circ &\circ &&\circ &\bullet \\
}}
\caption{Probability tree for Suzy and Billy.\label{fig:tree}}
\label{fig:tree}
\end{figure}

A probability tree of a theory $T$ in CP-logic defines an {\em a priori} probability distribution $P_T$ over all things that might happen in this domain, which can be read off the leaf nodes of the branches by multiplying the probabilities on the edges. For instance, the probability of the bottle breaking is the sum of the probabilities of the leaves in which $Breaks$ is $\true$ -- the white circles in Figure \ref{fig:tree} -- giving $0.98$. We have shown here only one such probability tree, but we can construct others as well by applying the laws in different orders. 

An important property however is that all trees defined by the same theory result in the same probability distribution. To ensure that this property holds even when there are bodies containing negative literals, CP-logic makes use of the well-founded semantics. Simply put, this means the condition for a law to be applied in a node is not merely that its body is currently satisfied, but that it will remain so. This implies that a negated atom in a body should not only be currently assigned $\false$, but actually has to have become impossible, so that it will remain $\false$ through to the end-state. For atoms currently assigned $\true$, it always holds that they remain $\true$, hence here there is no problem.

\subsubsection{Counterfactual Probabilities}

In the context of structural equations, \cite{pearl:book} studies counterfactuals and shows how they can be evaluated by means of a syntactic transformation. In their study of actual causation and explanations, \cite[p. 27]{halpernpearl05b} also define counterfactual probabilities (i.e., the probability that some event would have had in a counterfactual situation). \cite{vennekens:jelia} present an equivalent method for evaluating counterfactual probabilities in CP-logic, also making use of syntactic transformations. 

Assume we have a branch $b$ of a probability tree of some theory $T$. To make $T$ deterministic {\em in accordance with the choices made in $b$}, we transform $T$ into $T^b$ by replacing the heads of the laws that were applied in $b$ with the disjuncts that were chosen from those heads in $b$. For example, if we take as branch $b$ the previous story, then $T^b$ would be:

\begin{minipage}{0.4\textwidth}
\begin{align*}
Throws(Suzy) &\leftarrow.\\
Throws(Billy) &\leftarrow.
\end{align*}
\end{minipage}
\begin{minipage}{0.5\textwidth}
\begin{align*}
Breaks &\leftarrow Throws(Suzy).\\
Breaks &\leftarrow Throws(Billy).
\end{align*}
\end{minipage}
\\

We will use Pearl's $do()$-operator to indicate an intervention \cite{pearl:book}. The intervention on a theory $T$ that ensures variable $C$ remains false, denoted by $do(\lnot C)$, removes $C$ from the head of any law in which it occurs, yielding $T|do(\lnot C)$. For example, to prevent Suzy from throwing, the resulting theory $T|do(\lnot Throws(Suzy))$ is given by:

\begin{minipage}{0.4\textwidth}
\begin{align*}
&\leftarrow.\\
Throws(Billy) &\leftarrow.
\end{align*}
\end{minipage}
\begin{minipage}{0.5\textwidth}
\begin{align*}
(Breaks: 0.9) &\leftarrow Throws(Suzy).\\
(Breaks:0.8) &\leftarrow Throws(Billy).
\end{align*}
\end{minipage}
\\

Laws with an empty head are ineffective, and can thus simply be omitted. The analogous operation $do(C)$ on a theory $T$ corresponds to adding the deterministic law $C \leftarrow$.

With this in hand, we can now evaluate a Pearl-style counterfactual probability ``given that $b$ in fact occurred, the probability that $\lnot E$ would have occurred if $\lnot C$ had been the case'' as $P_{T^b}(\lnot E | do(\lnot C))$.

\section{Defining Actual Causation Using CP-logic}\label{sec:def}

We now formulate a general, parametrized definition of actual causation, which can accommodate several concrete definitions by filling in details that we first leave open. We demonstrate this using definitions by Hall and one by ourselves. For the rest of the paper, we assume that we are given a CP-theory $T$, an actual story $b$ in which both $C$ and $E$ occurred, and we are interested in whether or not $C$ caused $E$. By $Con$ we denote the quadruple $(T,b,C,E)$, and refer to this as a {\em context}.

\subsection{Actual Causation in General}

For reasons of simplicity, the majority of approaches (including Hall) only consider actual causation in a deterministic setting. Further, it is taken for granted that the actual values of all variables are given. In such a context, counterfactual dependence of the event $E$ on $C$ is expressed by the conditional: {\em if $do(\lnot C)$ then $\lnot E$}, where it is assumed that all exogenous variables take on their actual values. In our probabilistic setting, the latter translates into making those laws that were actually applied deterministic, in accordance with the choices made in the story. However in many cases some exogenous variables simply do not have an actual value to start with. For example, if Suzy is prevented from throwing her rock, then we cannot say what the accuracy would have been had she done so. In CP-logic, this would be represented by the fact that law \eqref{suzy} is not applied. Hence, in a more general setting, it is required only that $do(\lnot C)$ makes $\lnot E$ possible. In other words, we get a probabilistic definition of counterfactual dependence:

\begin{definition}[Dependence] $E$ is {\em counterfactually dependent} on $C$ in $(T,b)$ iff  $P_{T^b}(\lnot E | do(\lnot C)) > 0$.\label{def:cd}
\end{definition}  

As counterfactual dependency lies at the heart of causation for all of the approaches we are considering, Dependence represents the most straightforward definition of actual causation. It is however too crude and allows for many counterexamples, preemption being the most famous. 

More refined definitions agree with the general structure of the former, but modify the theory $T$ in more subtle ways than $T^b$ does. We identify two different kinds of laws in $T$, that should each be treated in a specific way. The first are the laws that are {\em intrinsic} with respect to the given context. These should be made deterministic in accordance with $b$. The second are laws that are {\em irrelevant} in the given context. These should simply be ignored. Together, the methods of determining which laws are intrinsic and irrelevant, respectively, will be the parameters of our general definition. Suppose we are given two functions $Int$ and $Irr$, which both map each context $(T,b,C,E)$ to a subset of the theory $T$. With these, we define actual causation as follows:

\begin{definition}[Actual causation given $Int$ and $Irr$] Given the context $Con$, we define that $C$ is an actual cause of $E$ if and only if $E$ is counterfactually dependent on $C$ according to the theory $T'$ that we construct as: 

$T' = [T \setminus (Irr(Con) \cup Int(Con))] \cup Int(Con)^b$. \label{def:actc}
\end{definition}

For instance, the naive approach that identifies actual causation with counterfactual dependence corresponds to taking $Irr$ as the constant function $\{\}$ and $Int(Con)$ as $\{ r \in T \mid r$ was applied in $b\}$. From now on, we use the following, more legible notation for a particular instantiation of this definition:

\begin{dependenceirrelevant}
No law $r$ is irrelevant.
\end{dependenceirrelevant}

\begin{dependenceintrinsic}
A law $r$ is intrinsic iff $r$ was applied in $b$. 
\end{dependenceintrinsic}

If desired, we can order different causes by their respective counterfactual probabilities, as this indicates how important the cause was for $E$.

\subsection{Beckers and Vennekens 2012 Definition}

A recent proposal by the current authors for a definition of actual causation was originally formulated in \cite{vennekens11}, and later slightly modified in \cite{beckers}. Here, we summarize the basic ideas of the latter, and refer to it as {\em BV12}. We reformulate this definition in order to fit into our framework. It is easily verified that both versions are equivalent.

Because we want to follow the actual story as closely as possible, the condition for intrinsicness is exactly like before: we force all laws that were applied in $b$ to have the same effect as they had in $b$.

To decide which laws were relevant for causing $E$ in our story, we start from a simple temporal criterion: every law that was applied after the effect $E$ took place is irrelevant, and every law that was applied before isn't. For example, to figure out why the bottle broke in our previous example, law \eqref{billy} is considered irrelevant, because the bottle was already broken by the time Billy's rock arrived. For laws that were not applied in $b$, we distinguish laws that could still be applied when $E$ occurred, from those that could not. The first are considered irrelevant, whereas the second aren't.  This ensures that any story $b'$ that is identical to $b$ up to and including the occurrence of $E$ provides the same judgements about the causes of $E$, since any law that is not applied in $b$ but is applied in $b'$, must obviously occur after $E$.

\begin{bv12irrelevant}
A law $r$ is irrelevant iff $r$ was not applied before $E$ in $b$, although it could have. (I.e., it was not impossible at the time when $E$ occurred.)
\end{bv12irrelevant}

\begin{bv12intrinsic}
A law $r$ is intrinsic iff $r$ was applied in $b$. 
\end{bv12intrinsic}

\subsection{Hall 2007}

One of the currently most refined concepts of actual causation is that of \cite{hall07}. Although Hall uses structural equations as a practical tool, he is of the opinion that intuitions about actual causation are best illustrated using neuron diagrams. A key advantage of these diagrams, which they share with CP-logic, is that they distinguish between a default and deviant state of a variable. Proponents of structural equations, on the other hand, countered Hall's approach by criticizing neuron diagrams' limited expressivity \cite[p. 398]{hitchcock09}. Indeed, a neuron diagram, and thus Hall's approach as well, is very limited in the kind of examples it can express. In particular, neuron diagrams can only express deterministic causal relations and they also lack the ability to directly express \emph{causation by omission}, i.e., that the absence of $C$ causes $E$. Hall's solution is to argue against causation by omission altogether. By contrast, we will offer an improvement of Hall's account that generalizes to a probabilistic context, and can also handle causation by omission. In short, we propose CP-logic as a way of overcoming the shortcomings of both structural equations and neuron diagrams.

In a neuron diagram, a neuron can be in one of two states, the default ``off" state and the deviant ``on" state in which the neuron ``fires". Different kinds of links define how the state of one node affects the other. For instance, in (a), $E$ fires iff at least one of $B$ or $D$ fires, $D$ fires iff $C$ fires, and $B$ fires iff $A$ fires and $C$ doesn't fire. Nodes that are ``on" are represented by full circles and nodes that are ``off" are shown as empty circles. 
\begin{center}
\hspace{\stretch{0.5}}
\parbox[c]{3cm}{
\begin{pgfpicture}
\pgfsetyvec{\pgfxy(0,1)}
\pgfnodecircle{a}[fill]{\pgfxy(0,1)}{0.25cm}
\pgfnodecircle{b}[stroke]{\pgfxy(1,1)}{0.25cm}
\pgfnodecircle{e}[fill]{\pgfxy(2,1)}{0.25cm}
\pgfnodecircle{c}[fill]{\pgfxy(0,0)}{0.25cm}
\pgfnodecircle{d}[fill]{\pgfxy(1,0)}{0.25cm}
\pgfsetendarrow{\pgfarrowto}
\pgfnodeconnline{a}{b}
\pgfnodeconnline{b}{e}
\pgfnodeconnline{c}{d}
\pgfnodeconnline{d}{e}
\pgfsetendarrow{\pgfarrowdot}
\pgfnodeconnline{c}{b}
\color{white}
\pgfputat{\pgfnodecenter{a}}{\pgfbox[center,center]{A}}
\pgfputat{\pgfnodecenter{c}}{\pgfbox[center,center]{C}}
\pgfputat{\pgfnodecenter{d}}{\pgfbox[center,center]{D}}
\pgfputat{\pgfnodecenter{e}}{\pgfbox[center,center]{E}}
\color{black}
\pgfputat{\pgfnodecenter{b}}{\pgfbox[center,center]{B}}
\end{pgfpicture}
\captionof{figure}{(a)}}
\hspace{\stretch{0.5}}
\parbox[c]{5.0cm}{
\begin{pgfpicture}
\pgfsetyvec{\pgfxy(0,1)}
\pgfnodecircle{a}[stroke]{\pgfxy(0,1)}{0.25cm}
\pgfnodecircle{b}[stroke]{\pgfxy(1,1)}{0.25cm}
\pgfnodecircle{e}[fill]{\pgfxy(2,1)}{0.25cm}
\pgfnodecircle{c}[fill]{\pgfxy(0,0)}{0.25cm}
\pgfnodecircle{d}[fill]{\pgfxy(1,0)}{0.25cm}
\pgfsetendarrow{\pgfarrowto}
\pgfnodeconnline{a}{b}
\pgfnodeconnline{b}{e}
\pgfnodeconnline{c}{d}
\pgfnodeconnline{d}{e}
\pgfsetendarrow{\pgfarrowdot}
\pgfnodeconnline{c}{b}
\color{white}
\pgfputat{\pgfnodecenter{c}}{\pgfbox[center,center]{C}}
\pgfputat{\pgfnodecenter{d}}{\pgfbox[center,center]{D}}
\pgfputat{\pgfnodecenter{e}}{\pgfbox[center,center]{E}}
\color{black}
\pgfputat{\pgfnodecenter{a}}{\pgfbox[center,center]{A}}
\pgfputat{\pgfnodecenter{b}}{\pgfbox[center,center]{B}}
\end{pgfpicture}
\captionof{figure}{(b)}}
\end{center}
Diagrams (a) and (b) represent the same causal structure, but different stories: in both cases there are two causal chains leading to $E$, one starting with $C$ and another starting with $A$. But in (a) the chain through $B$ is preempted by $C$, whereas in (b) there is nothing for $C$ to preempt, as $A$ doesn't even fire. Therefore (a) is an example of what is generally known as Early Preemption, whereas (b) is not.

Although Hall presents his arguments using neuron diagrams, his definitions are formulated in terms of structural equations that correspond to such diagrams in a straightforward way: for each endogenous variable there is one equation, which contains a propositional formula on the right concisely expressing the dependencies of the diagram.

Any structural model $M$ can also be read as a CP-logic theory $T$. The firing of the neurons and the resulting assignment to the variables in $M$, then correspond to a story $b$. 

One important difference between structural equations and CP-laws, is that we are not limited to using a single CP-law for each variable. As each law represents a separate causal mechanism, and only one mechanism can actually make a variable become $\true$, we will represent dependencies such as that of $E$ by three laws, corresponding to the three different ways in which $B$ and $D$ can cause $E$: each by themselves, or the two of them simultaneously. At first sight the conjunctive law may seem redundant, but if one has a temporal condition for irrelevance -- eg. BV12 -- then it may not be. The translation of examples (a) and (b) into CP-logic is given by the following CP-theory -- where $p$ and $q$ represent some probabilities:

\begin{minipage}{0.4\textwidth}
\begin{align*}
(A:p) & \leftarrow.\\
(C:q)& \leftarrow.
\end{align*}
\end{minipage}
\begin{minipage}{0.11\textwidth}
\begin{align*}
B& \leftarrow A\land \lnot C.\\
D& \leftarrow C.
\end{align*}
\end{minipage}
\begin{minipage}{0.5\textwidth}
\begin{align*}
E& \leftarrow B.\\
E& \leftarrow D.\\
E& \leftarrow B \land D.
\end{align*}
\end{minipage}

The idea behind Hall's definition is to check for counterfactual dependence in situations which are reductions of the actual situation, where a reduction is understood as ``a variant of this situation in which {\em strictly fewer} events occur". In other words, because the counterfactual dependence of $E$ on $C$ can be masked by the occurrence of events which are extrinsic to the actual causal process, we look at all possible scenario's in which there are less of these extrinsic events. Hall puts it like this \cite[p. 129]{hall07}:

\begin{quote}
Suppose we have a causal model for some situation. The model consists of some equations, plus a specification of the actual values of the variables. Those values tell us how the situation {\em actually} unfolds. But the same system of equations can also represent {\em nomologically possible variants}: just change the values of one or more exogenous variables, and update the rest in accordance with the equations. A good model will thus be able to represent a range of variations on the actual situation. Some of these variations will be -- or more accurately, will be modeled as -- {\em reductions} of the actual situation, in that every variable will either have its actual value or its default value. Suppose the model has variables for events $C$ and $E$. Consider the conditional 

\begin{center}{\bf if {\em C} = 0; then {\em E} = 0}\end{center}

This conditional may be true; if so, $C$ is a cause of $E$. Suppose instead that it is false. Then $C$ is a cause of $E$ iff there is a reduction of the actual situation according to which $C$ and $E$ still occur, and in which this conditional is true.
\end{quote}

Rather than speaking of fewer events occuring, in this definition Hall characterizes a reduction in terms of whether or not variables retain their actual value. This is because in the context of neuron diagrams, an event is the firing of a neuron, which is represented by a variable taking on its deviant value, i.e., the variable {\em becoming true}. In the dynamic context of CP-logic, the formal object that corresponds most naturally to Hall's informal concept of an event is the transition in a probability tree (i.e., the application of a causal law) that makes such a variable true. Therefore we take a reduction to mean that no law is applied such that it makes a variable true that did not become true in the actual setting.

To make this more precise, we introduce some new formal terminology. Let $d$ be a branch of a probability tree of the theory $T$. $Laws_d$ denotes the set of all laws that were applied in $d$. The resulting effect of the application of a law $r \in Laws_d$ -- i.e., the disjunct of the head which was chosen -- will be denoted by $r_d$, or by $0$ if an empty disjunct was chosen. The set of true variables in the leaf of $d$ will be denoted by $Leaf_d$.  

A branch $d$ is a {\em reduction} of $b$ iff $\forall r \in Laws_d: r_d = 0 \lor  \exists s \in Laws_b: r_d = s_b$. Or, equivalently, $Leaf_d \subseteq Leaf_b$.

A reduction of $b$ in which both $C$ and $E$ occur -- i.e., hold in its leaf -- will be called a $(C,E)$-reduction. The set of all of these will be denoted by $Red_b^{(C,E)}$. These are precisely the branches which are relevant for Hall's definition.
\begin{definition} \label{def5} We define that $C$ is an {\em actual cause} of $E$ iff 
$(\exists d \in Red_b^{(C,E)}: P_{T^{d}}(\lnot E | do(\lnot C))>0)$.
\end{definition}
Theorem 1 shows the correctness of our translation. Proofs of all theorems can be found in the Appendices.

\begin{theorem} Given a neuron diagram with its corresponding equations $M$, and an assignment to its variables $V$. Consider the CP-logic theory $T$ and story $b$ that we get when applying the translation discussed above. Then $C$ is an actual cause of $E$ in the diagram according to Hall's definition iff $C$ is an actual cause of $E$ in $b$ and $T$ according to Definition \ref{def5}.
\end{theorem}

At first sight, Definition \ref{def5} does not fit into the general framework we introduced earlier, because of the quantifier over different branches. However, we will now show that for a significant group of cases it actually suffices to consider just a single $T'$, which can be described in terms of irrelevant and intrinsic laws. 

Rather than looking at all of the reductions separately, we single out a minimal structure which contains the essence of our story. In general such a minimal structure need not be unique, as the story may contain elements none of which are necessary by themselves yet without all of them the essence is changed. The following makes this more precise.

\begin{definition}
A law $r$ is {\em necessary}  iff
\begin{itemize}
\item{$\forall d \in Red_b^{(C,E)}: r \in Laws_d$ and}
\item{$\forall d,e \in Red_b^{(C,E)}: r_d = r_e$.}
\end{itemize}
We define $Nec(b)$ as the set of all necessary laws.
\end{definition}

In general it might be that there are two (or more) edges which are unnecessary by themselves, but at least one of them has to be present. Consider for example a case where $C$ causes both $A$ and $B$, and each of those in return is sufficient to cause $E$. Then neither the law $r= A:... \leftarrow C$ nor the law $r' = B:... \leftarrow C$ is necessary, yet at least one of them has to be applied to get $E$. In cases where this complication does not arise, we shall say that the story is simple. 

\begin{definition}
A story $b$ is {\em simple} iff the following holds:
\begin{itemize}
\item $\forall r \in Laws_b:$ the head of $r$ contains at most two disjuncts;
\item $\forall d \in Red_b^{(C,E)}$, for all non-deterministic $r \in Laws_d \setminus Nec(b): \exists e \in Red_b^{(C,E)}$ so that $e = d$ up to the application of $r$, and $r_d \neq r_e$.
\end{itemize}
\end{definition}

As an example, note that the story in the previous paragraph is not simple. Neither law $r$ nor $r'$ is necessary. Now consider the $(C,E)$-reduction $d$ where first $r'$ fails to cause $B$, followed by $r$ causing $A$, which in turn causes $E$. The branch that is identical to $d$ up to and including the application of $r'$ but in which $r$ does not cause $A$, is not a $(C,E)$-reduction.

We are now in a position to formulate a theorem that will allow us to adjust Hall's definition into our framework.

\begin{theorem}
If $(\exists d \in Red_b^{(C,E)}: P_{T^{d}}(\lnot E | do(\lnot C))>0)$ then $P_{T^{Nec(b)}}(\lnot E | do(\lnot C))>0$. If $b$ is simple, then the reverse implication holds as well. 
\end{theorem}

It is possible to add an additional criterion to turn this theorem into an equivalence that also holds for non-simple stories. We choose not to do this, because all of the examples Hall discusses are simple, as are all of the classical examples discussed in the literarure, such as Early and Late Preemption, Symmetric Overdetermination, Switches, etc. 

As a result of this theorem, rather than having to look at all $(C,E)$-reductions and calculate their associated probabilities, we need only find all the necessary laws and calculate a single probability. If the story $b$ is simple, then this probability represents an extension of Hall's definition, since they are equivalent if one ignores the value of the probability but for it being $0$ or not. To obtain a workable definition of actual causation, we present a more constructive description of necessary laws. From now on we call the node resulting from the application of a law $r$ in $b$ $Node_r^b$. 

\begin{theorem}
If $b$ is simple, then a non-deterministic law $r$ is necessary iff there is no $(C,E)$-reduction passing through a sibling of $Node_r^b$.
\end{theorem}

With this result, we can finally formulate our version of Hall's definition, which we will refer to as Hall07.

\begin{hall07irrelevant}
No law $r$ is irrelevant.
\end{hall07irrelevant}

\begin{hall07intrinsic}
A law $r$ is intrinsic iff $r$ was applied in $b$, and there is no branch $d$ passing through a sibling of $Node_r^b$ such that $\{ C,E \} \subseteq Leaf_d \subseteq Leaf_b$.
\end{hall07intrinsic}

\subsection{Hall 2004 Definitions}

\cite{hall04} claims that it is impossible to account for the wide variety of examples in which we intuitively judge there to be actual causation by using a single, all-encompassing definition. Therefore he defines two different concepts which both deserve to be called forms of causation but are nonetheless not co-extensive. 

\subsubsection{Dependence}

The first of these is simply Dependence, as stated in Definition \ref{def:cd}. As mentioned earlier, Hall only considers deterministic causal relations, and thus the probabilistic counterfactual will either be $1$ or $0$.

\subsubsection{Production}

The second concept tries to express the idea that to cause something is to bring it about, or to {\em produce} it. The original, rather technical, definition can be found in the appendices, but the following informal version suffices for our purposes: $C$ is a producer of $E$ iff there is a directed path of firing neurons in the diagram from $C$ to $E$. In our framework, this translates to the following. 

\begin{productionirrelevant}
A law $r$ is irrelevant iff $r$ was not applied before $E$ in $b$, or if its effect was already $\true$ when it was applied.
\end{productionirrelevant}

\begin{productionintrinsic}
A law $r$ is intrinsic iff $r$ was applied in $b$. 
\end{productionintrinsic} 

\begin{theorem} \label{th1}
Given a neuron diagram with its corresponding equations $M$, and an assignment to its variables $V$. Consider the CP-logic theory $T$, and a story $b$, that we get when applying the translation discussed earlier. $C$ is a producer of $E$ in the diagram according to Hall iff $C$ is a producer of $E$ in $b$ and $T$ according to the CP-logic version stated here.
\end{theorem}

Besides providing a probabilistic extension, the CP-logic version of production also offers a way to make sense of causation by omission. That is, just as with all of the definitions in our framework in fact, we can extend it to allow negative literals such as $\lnot C$ to be causes as well. 

\section{Comparison}\label{sec:summary}

Table \ref{table:overview} presents a schematic overview of the four definitions discussed so far, as well as two new ones, that we give appropriate names. The columns and rows give the criteria for a law $r$ of $T$ to be considered intrinsic, respectively irrelevant, in relation to a story $b$, and an event $E$. By $r \leq_b E$, we denote that $r$ was applied in $b$ before $E$ occurred.

\begin{table}[ht] 
\caption{Spectrum of definitions} 
\centering 
\begin{tabular}{ c|  c  c} 
\hline\hline 
Irrelevant &  \multicolumn{2}{c}{Intrinsic} \\

 & $r \in Laws_b$ & $r \in Nec(b)$ \\ [0.5ex] 
\hline 
$\emptyset$ & Dependence  & Hall07\\
$\exists d: (d=b \text{ up to }E) \land r \geq_d E$& BV12 & BV07 \\ 
$r \nless_b E \lor r_b <_b r$ & Production & Production07 \\
\hline 
\end{tabular} 
\label{table:overview} 
\end{table} 

In order to illustrate the working of the definitions and to highlight their differences, we present an example:

\begin{quote}
$Assassin$ decides to poison the meal of a victim, who subsequently $Dies$ right before dessert. However, $Murderer$ decided to murder the victim as well, so he poisoned the dessert. If $Assassin$ had failed to do his job, then $Backup$ would have done so all the same. 
\end{quote}

The causal laws that form the context of this story are give by the following theory:

\begin{minipage}{0.4\textwidth}
\begin{align*}
(Assassin: p) &\leftarrow.\\
(Murderer: q) &\leftarrow.\\
(Backup: r) &\leftarrow \lnot Assassin.
\end{align*}
\end{minipage}
\begin{minipage}{0.5\textwidth}
\begin{align*}
Dies &\leftarrow Assassin.\\
Dies &\leftarrow Backup.\\
Dies &\leftarrow Murderer.
\end{align*}
\end{minipage}
\\ \\
In this story, did $Assassin$ cause $Dies$? We leave it to the reader to verify that in this case the left intrinsicness condition from the table applies to the first two non-deterministic laws, whereas the right one only applies to the first. The second irrelevance condition only applies to the last law, whereas the third one applies to the last two laws and to the third. This results in the following probabilities representing the causal status of $Assassin$:\\

\begin{tabular}{ c  c  c  c}
\hline
 Production & BV12 & Hall07 & Dependence\\
\hline
$1$ & $1-r$ & $(1-r)*(1-q)$ & $0$\\
\hline 
\end{tabular} 
\\ \\
Different motivations can be provided for these answers:

\begin{itemize}
\item {\bf Production}:  $Assassin$ brought about the death of the victim all by himself, hence he is the full cause.
\item {\bf BV12}: If $Assassin$ hadn't killed him, then that omission itself would not have lead to victim's death with a probability of $(1-r)$. Hence, $Assassin$ is a cause of the death to this extent.
\item {\bf Hall07}: Ignoring the actually redundant $Murderer$, if $Assassin$ doesn't kill him, then there is a $(1-r)*(1-q)$ probability that the victim will die. Hence he is the cause to that extent.
\item {\bf Dependence}: The victim would have died anyway, so $Assassin$ is not a cause at all.
\end{itemize}

Rather than saying that only one of these answers is correct, we prefer to think of them as answering different questions, all of which have their use in some context or other. (Eg., to determine responsibility, understand Assassin's state of mind, minimize the chance of murders, etc.) More generally, the definitions could be characterized by describing which events are allowed to happen in the counterfactual worlds they take into consideration to judge causation. 

\begin{itemize}
\item {\bf Production}: Only those events -- i.e., applications of laws -- which led to $E$, and not differently -- i.e., with the same outcome as in the actual story.
\item {\bf BV12}: Those events which led to $E$, and not differently, and also those events which were prevented from happening by these.
\item {\bf Hall07}: Any event can happen, as long as those events that were essential to lead to $E$ do not happen differently.
\item {\bf Dependence}: Any event can happen, as long as those events that did actually happen do not happen differently.
\end{itemize}

\section{Conclusion}

In this paper we have used the formal language of CP-logic to formulate a general definition of actual causation, which we used to express four specific definitions: a proposal of our own, and three definitions based on the work of Hall. By moving from the deterministic context of neuron diagrams to the non-deterministic context of CP-logic, the latter definitions improve on the original ones in two ways: they can deal with a wider class of examples, and they allow for a graded judgment of actual causation in the form of a conditional probability. Further, comparison between the definitions is facilitated by presenting them as various ways of filling in two central concepts. We have illustrated the flexibility of CP-logic in expressing different definitions, opening the path to other proposals beyond the ones here discussed.

\section{Appendices}

To facilitate the proof of the first theorem, we introduce the following lemma.

\begin{lemma} Given a neuron diagram $D$ with its corresponding equations $M$, and an assignment to its variables $V$. Consider the CP-logic theory $T$, and a story $b$, that we get when applying the translation discussed earlier. Then a neuron diagram $R$ is a reduction of $D$ in which both $C$ and $E$ occur iff its translation $d$ -- another branch of $T$ -- is a $(C,E)$-reduction of $b$.
\end{lemma}

\begin{proof}
Assume we have a reduction $R$ of a neuron diagram $D$, and $b$ is the story corresponding to $D$. As $R$ is simply a different assignment of the variables occurring in $D$, brought about by the same equations that existed for $D$, this reduction corresponds to another branch $d$ of $T$, in which $C$ and $E$ hold in its leaf.  Moreover, $R$ can be constructed starting from $D$ by changing some of the exogenous variables, say $U'$, from their actual values to their default value, and then updating the endogenous variables in accordance with the deterministic equations. It being a reduction, this caused no new variables to take on their deviant value in comparison to $D$. Let $r$ be a law that occurs in $d$.

If $r$ is non-deterministic, it must be one of the laws representing an exogenous variable $V$, i.e., a law with an empty body, and hence it was also applied in $b$. $R$ being a reduction, either $V$ has the same value in $R$ as in the original diagram, or it has its default value. In the former case, this means that $r_d = r_b$, in the latter case $r_d = 0$, both of which satisfy the requirement for $d$ being a reduction.

If $r$ is deterministic, the precondition for $r$ has to be fulfilled in $d$, causing some variable $V$ to take on its deviant value. The same must hold true of the precondition for the equation for $V$, and thus $V$ takes on its deviant value in $R$ as well, implying it did so in $D$ too. Therefore there must have been some law applied in $b$ that made $V$ take on its deviant value as well. From this it follows that $d$ is a $(C,E)$-reduction of $b$. 

Now assume we have a theory $T$ and a story $b$ that form the translation of a neuron diagram $D$, such that $C$ and $E$ hold in $b$, and that $d$ is a $(C,E)$-reduction of $b$. As the leaf of $d$ contains an assignment to all of the variables that satisfies the equations of $M$, there is a neuron diagram $R$ that corresponds to $d$. We can easily go over all the previous steps in the other direction, to conclude that $R$ is a reduction of $D$ in which $C$ and $E$ are true. 
\end{proof}

\setcounter{theorem}{1}
\addtocounter{theorem}{-1}

\begin{theorem} Given a neuron diagram with its corresponding equations $M$, and an assignment to its variables $V$. Consider the CP-logic theory $T$ and story $b$ that we get when applying the translation discussed above. Then $C$ is an actual cause of $E$ in the diagram according to Hall's definition iff $C$ is an actual cause of $E$ in $b$ and $T$ according to Definition \ref{def5}.
\end{theorem}

\begin{proof}
We start with the implication from left to right. Assume we have a neuron diagram $D$, in which both $C$ and $E$ fire. This translates into a theory $T$ and a story $b$, for which $C$ and $E$ hold in its leaf. Further, assume there is a reduction $R$ of this diagram, in which both $C$ and $E$ continue to hold, and in this reduction, if {\em C} = 0; then {\em E} = 0. By the above lemma, this translates into a $(C,E)$-reduction of $b$, say $d$. 

In $R$, if {\em C} = 0; then {\em E} = 0. The conditional $C = 0$ is interpreted as a counterfactual locution, and corresponds to $do(\lnot C)$. As there are no non-deterministic laws with non-empty preconditions, $T^d$ is simply the deterministic theory that determines the same assignment as $R$, meaning $P_{T^{d}}(\lnot E | do(\lnot C))=1$, which concludes this part of the proof.

Now assume we have a theory $T$ and a story $b$ that form the translation of a neuron diagram $D$, such that $C$ and $E$ hold in $b$, and that $d$ is a $(C,E)$-reduction of $b$ for which the given inequality holds. By the above lemma, the translation of $d$, say $R$, is reduction of $D$ in which $C$ and $E$ occur. As mentioned in the previous paragraph, $T^d$ simply corresponds to an assignment of values to the variables occurring in $D$ that follows its equations. Since $R$ describes this same assignment, in $R$ too if {\em C} = 0; then {\em E} = 0. This concludes the proof. 
\end{proof}

\begin{theorem}
If $(\exists d \in Red_b^{(C,E)}: P_{T^{d}}(\lnot E | do(\lnot C))>0)$ then $P_{T^{Nec(b)}}(\lnot E | do(\lnot C))>0$. If $b$ is simple, then the reverse implication holds as well. 
\end{theorem}

\begin{proof}
We start with proving the first implication. Assume we have a $d \in Red_b^{(C,E)}$ such that $P_{T^{d}}(\lnot E | do(\lnot C))>0$. This implies that there is at least one branch $e$ of a probability tree of $T^{d}|do(\lnot C)$ for which $\lnot E$ holds in its leaf. We prove by induction on the length of $e$ that this implies the existence of a similar branch $e'$ of a probability tree of $T^{Nec(b)}|do(\lnot C)$ for which $\lnot E$ holds in its leaf, which is what is required to establish the theorem. 

Base case: if $e$ consists of a single node -- i.e., the root node where all atoms are false -- then this means that no laws of $T^{d}|do(\lnot C)$ can be applied. Since the bodies of the laws in $T^{Nec(b)}|do(\lnot C)$ are identical to those of the laws in $T^{d}|do(\lnot C)$, we simply have $e' = e$. 

Induction case: Assume we have a sub-branch $e_n$ of $e$ with length $n > 1$, starting from the root node, and that we also have a structurally identical sub-branch $e'_n$. By it being structurally identical we mean that they are identical except for the fact that they may have different probabilities along the edges.

If $e_n = e$, then no more laws can be applied in the final node of $e_n$. This must then hold for the final node of $e'_n$ as well, so we are finished. Otherwise, we know that there is a sub-branch $e_{n+1}$ which extends $e_n$ along $e$ with a node $O$. Assume that the law which was applied to get to $O$ is $r$. 

If $r$ is deterministic, then $r$ occurs in $T^{d}|do(\lnot C)$ exactly as it does in $T^{Nec(b)}|do(\lnot C)$. Since both branches are structurally identical, $e'_n$ can be extended in the exact same manner as $e_n$, so there has to be a probability tree of $T^{Nec(b)}|do(\lnot C)$ in which there is a sub-branch $e'_{n+1}$ with the desired properties. So assume $r$ is non-deterministic. 

First assume $r \not \in Laws_d$. This implies that $r \not \in Nec(b)$. So as in the deterministic case, $r$ occurs in $T^{d}|do(\lnot C)$ exactly as it does in $T^{Nec(b)}|do(\lnot C)$, and the branch can be extended in the same manner.

Now assume $r \in Laws_d$. If also $r \in Nec(b)$, we know that $r_d = r_b = r_{Nec}$ and hence the previous argument holds. Remains the possibility that $r \not \in Nec(b)$. As in the deterministic case, because $r$ can be applied in the final node of $e'_n$ there has to be a probability tree of $T^{Nec(b)}|do(\lnot C)$ with a sub-branch like $e'_n$ where $r$ is applied next. 

Assume $r_d = A$. Since $A$ was the outcome of $r$ in $d$, the law $r$ as it appears in $T$ -- and also in $T^{Nec(b)}|do(\lnot C)$ -- contains $A$ in its head with some probability attached to it. Therefore the final node of $e'_n$ in the said probability tree has one child-node which contains $A$, extending $e'_n$ into a sub-branch $e'_{n+1}$ with the desired properties. This concludes this part of the proof. 

Now we prove that if $b$ is simple, the reverse implication holds as well. 

Assume $P_{T^{Nec(b)}}(\lnot E | do(\lnot C))>0$. This implies that there is at least one branch $e$ of a probability tree of $T^{Nec(b)}|do(\lnot C)$ for which $\lnot E$ holds in its leaf. We can repeat the first steps of the previous implication, so that we again arrive at a law $r$ which was applied to get to a node $O$.

The branch $e'$ we are considering occurs in a probability tree of a $(C,E)$-reduction, say $f$. First assume $r \in Nec(b)$. By definition, this implies that also $r \in Laws_f \land r_{Nec} = r_f$, and we can apply the reasoning from above. Likewise as above, we can apply this reasoning to all other cases, except the one where $r \not \in Nec(b)$, $r$ is non-deterministic, and $r \in Laws_f$. Assume the law $r$ has effect $A$ in the branch $e$ we are considering. If $r_f = A$, then we are back to our familiar situation, so therefore assume $r_f = B$, and $A \neq B$. 

Since $b$ is simple, $A$ and $B$ are the only two possible effects of $r$. Further, remark that $r \in Laws_b \setminus Nec(b)$. This implies the existence of a $(C,E)$-reduction $g$ that is identical to $f$ up to the application of $r$, but such that $r_g \neq r_f$, and thus $r_g = A = r_e$ meaning there is a branch in a probability tree of $g$ that is structurally identical to $e$ up to $O$. This concludes the proof of the theorem.
\end{proof}

\begin{theorem}
If $b$ is simple, then a non-deterministic law $r$ is necessary iff there is no $(C,E)$-reduction passing through a sibling of $Node_r^b$.
\end{theorem}

\begin{proof}
Say the unique sibling of $Node_r^b$ is $M$. We start with the implication from left to right, so we assume $r$ is necessary. Assume $r_b = A$, then there is no $d \in Red_b^{(C,E)}$ for which $r_d \neq A$, hence there is no $(C,E)$-reduction which passes through $M$.

Remains the implication from right to left. Assume we have a law $r$ such that there is no $(C,E)$-reduction passing through a sibling of $Node_r^b$. We proceed with a reductio ad absurdum, so we assume $r$ is not necessary. 

Clearly $b$ is a $(C,E)$-reduction of itself, and also $r \in Laws_b \setminus Nec(b)$. Hence, by $b$'s simplicity, there is a $(C,E)$-reduction $e$ which is identical to $b$ up to the application of $r$, but for which $r_e \neq r_b$. Thus $e$ passes through the sibling of $Node_r^b$, contradicting the assumption that $r$ is necessary. This concludes the proof. 
\end{proof}

\begin{theorem}
Given a neuron diagram with its corresponding equations $M$, and an assignment to its variables $V$. Consider the CP-logic theory $T$, and a story $b$, that we get when applying the translation discussed earlier. $C$ is a producer of $E$ in the diagram according to Hall iff $C$ is a producer of $E$ in $b$ and $T$ according to the CP-logic version stated here.
\end{theorem}

\begin{proof}
First we need to explain some terminology that Hall uses. A \emph {structure} is a temporal sequence of sets of events, which unfold according to the equations of some neuron diagram. A branch, or a sub-branch, would be the corresponding concept in CP-logic. 

Two structures are said to \emph{match intrinsically} when they are represented in an identical manner. The reason why Hall uses this term, is because even though we use the same variable for an event occurring in different circumstances, strictly speaking they are not the same. This is mainly an ontological issue, which need not detain us for our present purposes.

A set of events $S$ is said to be {\em sufficient} for another event $E$, if the fact that $E$ occurs follows from the causal laws, together with the premisse that $S$ occurs at some time $t$, and no other events occur at this time. A set is {\em minimally} sufficient if it is sufficient, and no proper subset is. To understand this, note that the ambiguity of the relation between an event and the value of a variable that we noted earlier, resurfaces here. In the context of neuron diagrams, events are temporal, and occur during the time-period that a neuron fires, i.e, becomes true. However, at any later time-point, the variable corresponding to this neuron will remain to be true, implying that the value of the variable has shifted in meaning from ``the neuron fires" to ``the neuron has fired". Given this interpretation, it is natural to translate Hall's notion of an event into CP-logic as the application of a law, making a variable true, as we have done.

A further detail to be cleared out, is that in the context of neuron diagrams there can be simultaneous events, since multiple neurons can fire at the same time. In CP-logic, in each node only one law is allowed to be applied, hence this translates to two consecutive edges in a branch. Therefore it is not the case that each node-edge pair in a branch corresponds to a separate time-point, but rather sets of consecutive pairs -- with variable size -- do. Given such a set, then for each variable that was the result of the application of a law belonging to it, it holds that its corresponding event occurs at the next time-point, corresponding to the next set of nodes further down the branch. All the variables occuring in the bodies of the laws in this set, represent events that occur during this time-point.

Now we can state the precise definition of production as it occurs in \cite[p.25]{hall04}. 

\begin{quote}
We begin as before, by supposing that $E$ occurs at $t'$, and that $t$ is an earlier time such that at each time between $t$ and $t'$, there is a unique minimally sufficient set for $E$. But now we add the requirement that whenever $t_0$ and $t_1$ are two such times ($t_0 < t_1$) and $S_0$ and $S_1$ the corresponding minimally sufficient sets, then
\begin{itemize}
\item for each element of $S_1$, there is at $t_0$ a unique minimally sufficient set; and
\item the union of these minimally sufficient sets is $S_0$.
\end{itemize}
...

Given some event $E$ occurring at time $t'$ and given some earlier time $t$, we will say that $E$ has a {\emph pure causal history} back to time $t$ just in case there is, at every time between $t$ and $t'$, a unique minimally sufficient set for $E$, and the collection of these sets meets the two foregoing constraints. We will call the structure consisting of the members of these sets the ``pure causal history" of $E$, back to time $t$.
We will say that $C$ is a proximate cause of $E$ just in case $C$ and $E$ belong to some structure of events $S$ for which there is at least one nomologically possible structure $S'$ such that (i) $S'$ intrinsically matches $S$; and (ii) $S'$ consists of an $E$-duplicate, together with a pure causal history of this $E$-duplicate back to some earlier time. (In easy cases, $S$ will itself be the needed duplicate structure.)
Production, finally, is defined as the ancestral [i.e., the transitive closure] of proximate causation.
\end{quote}

We will start with the implication from left to right. So assume we have a neuron diagram $D$, in which $C$ is a producer of $E$. Say $T$ is the CP-logic theory that is the translation of the equations of the diagram, and $b$ is the branch representing the story. We already know that $C$ and $E$ hold in the leaf of $b$. We need to proof that $P_{T'}(\lnot E | do(\lnot C)) > 0$. The theory $T'$ only contains deterministic laws, and no disjunctions, hence all its laws are of the form: $V \leftarrow A \land A' \land  ... \land \lnot B \land \lnot B'$, where the number of positive literals in the conjunction is at least one. Therefore any probability tree for $T'$ consists out of only one branch, determining a unique assignment for all the variables. Further, even though the theory $T$ may contain several laws in which a variable occurs in the head, because of our irrelevance criterion $T'$ contains exactly one law for every variable that is true. So for every true variable in this assignment, there is a unique chain of laws -- neglecting the order -- which needs to be applied to make this variable true. For any such variable $V$, we will say that it depends on all of the variables occurring positively in the body of a law in this chain. Clearly, if any true variable changes its value in this assignment, then all variables which depend on it become false.

As a first case, assume $C$ is a proximate cause of $E$. We start by assuming that circumstances are nice, meaning that $D$ contains itself a structure $S$ which is a pure causal history of $E$. This means that in the actual story $b$, $C$ is part of a unique minimally sufficient set for $E$. From this it follows that in $T'$, $C$ figures positively in one of the laws on which $E$ depends. Hence, if we apply $do(\lnot C)$, then $E$ will no longer hold. 

Now assume that there is a structure $S$ occurring in $D$, such that there exists another diagram, say $D'$, in which this structure occurs as well, and forms a pure causal history of $E$. This diagram corresponds to a branch of $T$, say $d$, That means that in $T'_d$ -- i.e., the theory $T'$ constructed out $d$ -- $C$ occurs positively in the unique chain of laws which can make $E$ true. But as all events in $S$ also occur in $D$, at the same moments as they do in $D'$, that means that $C$ must also occur positively in the unique chain of laws for $E$ in the theory $T'_b$. Hence, $E$ depends on $C$ in the theory $T'_b$ as well.

Now look at the more general case, in which $C$ occurs in a chain of proximate causes, that leads up to $E$. I.e, in $D$, $C$ is the proximate cause of some variable $V_1$, which in turn is the proximate cause of some variable $V_2$, and so on until we get to $E$. We know from the previous discussion, that this implies in $T'$ that $do(\lnot C)$ then $\lnot V_1$, and $do(\lnot V_1)$ then $\lnot V_2$, and so on. Given what we know about $T'$, it directly follows that when we apply $do(\lnot C)$, then $\lnot E$. This concludes this part of the proof.

We continue with the implication from right to left. So assume that we are given again a neuron diagram and a corresponding story $b$, and that we know $P_{T'}(\lnot E | do(\lnot C)) > 0$. From our earlier analysis of $T'$, we know that this means that $C$ occurs positively in the unique chain of laws that can make $E$ true according to $T'$. From this chain of laws, we start from the one causing $E$ and from there pick out a series that gets us to a law where $C$ occurs positively in the body. More concretely, we take a series of the form: $E \leftarrow ... A \land ... $, $A \leftarrow ... D \land ...$, and so on until we get at a law $Z \leftarrow ... C \land ...$. By definition of production, it suffices to prove that in this chain, each of the variables in the body is a proximate cause of the variable in the head. 

Take such a law $V \leftarrow ... W \land ...$. At the time that this law is applied, $W$ clearly is a member of a sufficient set of events for $V$, which occurs at the next time point. Say $S_0$ is the set of all events that occur together with $W$ that figure in the body of this law, and $S_1$ is the set $\{ V \}$ that occurs at the next time-point, then the structure consisting precisely of $S_0$ and $S_1$ and nothing else forms a pure causal history of $V$ containing $W$. The same reasoning applies to all laws of the chain. This concludes the proof.
\end{proof}

\bibliography{ijcai15_mybib}

\end{document}